\documentclass[aps,prx,reprint,amsmath,amssymb,longbibliography]{revtex4-2}

\usepackage{graphicx}   
\usepackage{bm}         
\usepackage{amsfonts}
\usepackage{amsthm}     
\usepackage{comment}    
\usepackage{booktabs}   
\usepackage{tabularx}   
\usepackage{array}      
\usepackage[table]{xcolor}
\usepackage{colortbl}   
\usepackage{hyperref}   
\usepackage{dcolumn}
\usepackage{soul}
\usepackage{tikz}

\hypersetup{
  colorlinks=true,
  linkcolor=blue,
  urlcolor=cyan,
  citecolor=blue
}

\newtheorem{theorem}{Theorem}[section]
\newtheorem{proposition}[theorem]{Proposition}
\newtheorem{definition}[theorem]{Definition}
\newtheorem{corollary}[theorem]{Corollary}

\DeclareMathOperator{\sgn}{sgn}

\begin{document}

\title{Machine-learnable Sets}

\author{Veit Elser}
\thanks{ve10@cornell.edu}
\affiliation{Department of Physics, Cornell}
\author{Manish Krishan Lal}
\thanks{manish.krishanlal@gmail.com}
\affiliation{Mathematics, Independent Researcher, India}

\date{\today}

\begin{abstract}
In this study we present a formal definition of large discrete sets having, informally, three properties: their elements are easily recognized, easily generated, and the latter tasks are easily learned from examples. The formalism is specialized to sets of binary strings and a definition of ``machine-learnability'' based on the existence of a bounded-complexity Boolean autoencoder that fixes the elements of the set. We present experiments where the autoencoders are implemented by nets of Boolean threshold functions. Machine-learnability is demonstrated for Rorschach patterns (that may have reversed contrast in the mirrored half), and considerably ``wilder'' sets whose elements are only approximately fixed by admissible autoencoders. In the second case we demonstrate a simple iteration that evolves wild sets to make them properly machine-learnable.
\end{abstract}

\maketitle

\section{Introduction}

Formal models of machine learning are used to establish what kinds of learning are possible in principle, and to provide a framework for evaluating its practice. ``Language identification in the limit,'' by Gold \cite{gold1967language}, is the best known example of the former and launched the field of algorithmic learning theory. The task considered was learning the grammar of an abstract language from examples, with computations limited only by a Turing machine. Valiant's \cite{valiant1984theory} \textit{Probably Approximately Correct} (PAC) learning model moved the focus in the direction of practice. Computations in the PAC model have polynomial complexity bounds, and exact learning is replaced by probabilistic guarantees. In this paper we move the needle even further, with artificial neural networks imposing explicit bounds on the computations. The model, called \textit{machine-learnable sets}, links learnability directly to the machine doing the learning. 

Language acquisition---in humans, not machines---is the example of learnable sets we are most familiar with. The set in question is the set $X$ of all grammatical, bounded-length sentences in some language. A notable feature of $X$ is that listing its elements is not feasible. In spite of this, $X$ has the following nice properties: 
\begin{enumerate}

\item Given a string (of symbols) $x$, it is easy to decide whether $x\in X$.

\item It is easy to uniformly sample all of $X$.

\item Though somewhat harder, doing tasks 1 and 2 can be learned from positive examples $x\in X^*\subset X$, where $|X^*|\ll |X|$.

\end{enumerate}
Less formally, a child of even distracted parents easily rises to the level of a grammatical chatterbox!
Language is of course much more than what these properties try to capture, and we make no claims that something like this lies at the heart of natural language. Still, as we explain next, language suggests a mechanism 
---something that a machine can implement---for realizing sets $X$ having the above properties.

Behind (above? under?) every grammatical sentence is its meaning. Speakers of a language have equal facility at transforming sentences into meanings and meanings into sentences. We tend to think of meanings as shared across different languages, with English and Hindi differing only in the transformations to and from sentences. We can formalize this way of looking at language with the autoencoder architecture of machine learning. The autoencoder receives a sentence in some language, easily transforms it into its compact, universal meaning, and just as easily transforms the meaning back to the original sentence.

The representation of sentences by symbols is well known, the representation of meanings much less so. If we represent both as strings of bits, then because the transformations in the autoencoder are deterministic, the two kinds of strings have equal entropy. An interesting case is where the sentence-strings are longer than the meaning-strings by some factor greater than two, and the length of the meaning-strings is close to the entropy of strings (of either kind) in bits. An autoencoder with these characteristics can act as the identity on only a very special subset of the sentence strings---those having an associated meaning-string---and these define the set $X$ in the list above. To perform task 1 one simply inputs $x$ and checks that $x$ comes out. The sentence-strings generated by the output half of the autoencoder, from freely sampled meaning-strings, accomplishes task 2. Property 3 in the list above is simply the observation that if the complexity of the autoencoder can be suitably bounded, then the information provided by a modest number of sentence-examples should suffice to reconstruct it uniquely.

The language example compels us to reverse the usual designations of the two autoencoder halves. Meanings are primary, and languages are the codes that meanings are encoded into and decoded from. In this paper, the encoder is always the output-half of the autoencoder. 

Language is also our motivation behind shifting attention from the functions that act on sets to the sets themselves. One of the legacies of MNIST is a
``they are what they are ...'' attitude about data, and to not let that discourage the pursuit of perfect classification. While this is the correct attitude in many if not most applications of machine learning, in the case of language we might question whether the data (e.g. a grammar) deserves the same unquestioned degree of permanence.

Suppose $X_0$ is a proto-language and speakers in generation $0$ learn it poorly---as an autoencoder $\mathcal{A}_0$---because of its faults. That is, $\mathcal{A}_0(X_0)=X_1$, where $X_1$ is close to but distinct from the proto-language $X_0$. As the representation of $X_1$ grows in the population, speakers of generation $1$ will start to learn language $X_1$, again imperfectly, as $\mathcal{A}_1$ and introduce yet another variant, $\mathcal{A}_1(X_1)=X_2$. And so on, where the refinements become smaller and the languages easier to learn. Spanish might not just be the easiest language to learn, but the most evolved!

The \textit{machine} qualifier in machine-learnable sets is what sets our learning model apart from previous models. Polynomial time-space limits on computations, as in PAC learning, seem extravagant compared to processing data by passing it through a neural network. A scaled-down model of computation also aligns with our experience of language acquisition: an automatic process distinct from ``cogitation.'' The autoencoder implementation of tasks 1 and 2 clearly fits this description. Until recently, task 3, which concerns the training of neural networks, could not be addressed because discrete representations and circuit (encoder/decoder) reconstructions are difficult for gradient-based optimization. However, first results \cite{elser2026learning} with a constraint-based method that trains networks of Boolean threshold functions showed that these limitations no longer apply.

The remainder of this paper is organized as follows. Related previous work is reviewed in Section \ref{sec:previous}. Machine-learnable sets are formally defined in Section \ref{sec:MLS} in application-blind terms. The meaning-strings of language just comprise another set called $Y$. Section \ref{sec:MLS} also presents the theory behind property 3 (a form of generalization) and a heuristic argument why the set-migration process described above, when iterating imperfect auto\-encoders, may converge.

Section \ref{sec:BTF} gives a brief review of the constraint-based approach for training networks whose neurons---Boolean threshold functions---have strict $\pm 1$ outputs. From the user side, working with \textit{boolnets} is not all that different from working with standard networks. Instead of minimizing loss, the learning algorithm ``closes the gap.''

Sections \ref{sec:RS} and \ref{sec:wild} present machine learning experiments with perfect and imperfect machine-learnable sets. In the first case, the autoencoder very quickly realizes it is dealing with Rorschach patterns when the pixel values conveyed to the network are forced to flow through a layer having half as many nodes as pixels. The task is made more challenging by flipping the contrast of the mirrored pixels in some of the data.

The machine-learnable sets in the second case (Sec. \ref{sec:wild}) are ``wild'' in that there is no simple principle, like Rorschach symmetry, in their definition. We first consider a set $X_0$ created by a random encoder circuit for which a decoder---even its architecture---is unknown. In the second example, where $X_0$ is down-sampled MNIST, not even the existence of an encoder is known. However, experiments show that refinements of $X_0$ by the language-inspired iteration scheme, all computed with the same autoencoder architecture, become increasingly machine-learnable.

The conspicuous absence of statistical learning theory (or established theory of any kind) in this work is addressed in Section \ref{sec:discussion}. Whereas guarantees of one form or another will always be valuable in the applications of machine learning, the more investigative approach taken in this paper provides an important complement.

\section{Related work}\label{sec:previous}

We were astonished to discover that several of the ideas that inspired our work had previously been advanced by Kirby et al. \cite{kirby2001iterated,kirby2008cumulative}. Not just the idea of agents learning with shared hardware as the basis for the evolution of structured data (language), but the more technical detail presented by ``the problem of being transmitted through a narrow bottleneck'' \cite{kirby2008cumulative}. Our work can be viewed as a formal distillation of these ideas, including demonstrations on structured data of much lower complexity than language.

As explained more fully in Section \ref{sec:discussion}, our methods lie outside the usual statistical-learning framework. We are able to dispense with distributions by the trick of forcing all activations to be discrete, and learning which latent configurations to reject. Variational autoencoders whose latent variables are discrete (through vector quantization) were developed by van den Oord et al. \cite{vandenoord2017neural}. This work also showed it is possible to learn distributions on the latent variables with a modified gradient-based optimizer. Our constraint-based optimizer allows \textit{all} neuron outputs to be discrete---a property needed if there is any hope of bounding the autoencoder's complexity.

\section{Machine-learnable sets}\label{sec:MLS}

Machine-learnable sets are sets of Boolean strings. Before defining them we define sets of functions in this domain.

\begin{definition}
$\mathcal{B}(m,n,k)$ is a class of the Boolean functions $\{0,1\}^m\to \{0,1\}^n$ that can be specified with $k$ bits of information. 
The integers $m$ and $n$ are called the input and output dimensions.
\end{definition}
\noindent In the next section, which covers the  implementation of Boolean functions, we will specialize to a particular class of functions. If you need something more concrete in the meantime, think of the $k$ in $\mathcal{B}(m,n,k)$ as a log-bound on the number of gates in a circuit. In general,  $k_1< k_2$ implies $\mathcal{B}(m,n,k_1)\subset \mathcal{B}(m,n,k_2)$.

As explained in the introduction, decoders and encoders have reversed roles in our definition of an autoencoder. Moreover, our definition of machine-learnable sets is interesting only when the dimension $m$ is so large that storing $2^m$ items is not feasible (and similarly for the dimension $n$).  
\begin{definition}\label{def:MLS}
A machine-learnable set $X\subset\{0,1\}^m$ of internal dimension $n<m$ is defined by the property
\begin{equation}
x\in X \iff \mathcal{A}(x)=x,
\end{equation}
where the autoencoder
\begin{equation*}
\mathcal{A}=\mathcal{E}\circ\mathcal{D}:\;\{0,1\}^m\to\{0,1\}^m\;,
\end{equation*}
is the composition of a decoder
\begin{equation*}
\mathcal{D}\in \mathcal{B}(m,n,k_\mathcal{D})\;,
\end{equation*}
and an encoder
\begin{equation*}
\mathcal{E}\in \mathcal{B}(n,m,k_\mathcal{E})\;.
\end{equation*}
\end{definition}

\begin{definition}\label{def:accept}
The acceptance of the autoencoder $\mathcal{A}=\mathcal{E}\circ\mathcal{D}$ is the probability
\begin{equation}\label{eq:accept}
\alpha=\underset{y\in\{0,1\}^n}{\mathrm{prob}}\Big(\mathcal{A}(\mathcal{E}(y))=\mathcal{E}(y)\Big)\;,
\end{equation}
where the $y$'s are sampled uniformly.
\end{definition}

\noindent These definitions, and access to the decoder $\mathcal{D}$ and encoder $\mathcal{E}$, account for the properties 1 and 2 of machine-learnable sets stated in the introduction. Membership testing, by checking if $x\in\{0,1\}^m$ is fixed by $\mathcal{A}$, is easy because the autoencoder has a complexity bound. By the second definition, the membership test applied to $\mathcal{E}(y)$, where $y\in\{0,1\}^n$ is uniformly sampled, succeeds with probability $\alpha$ or in $1/\alpha$ trials. Provided $\alpha$ is not very small, this too is easy. For property 3 (examples needed for learning) we need to know something about the size of $X$ (Section \ref{sec:infosuff}):

\begin{corollary}
When the encoder of the autoencoder $\mathcal{A}$ that defines the machine-learnable set $X$ is injective,
\begin{equation}
|X|=\alpha\, 2^n,
\end{equation}
where $\alpha$ and $n$ are the acceptance and internal dimension of $\mathcal{A}$.
\end{corollary}

\subsection{Encoder injectivity}

Encoder injectivity is in general difficult to establish. However, when $m$ is large and $m>2n$, then even a random encoder has a very high probability of being injective. If $M=2^m$ and $N=2^n$, this probability equals
\begin{equation*}
p_\text{inj}=\frac{M(M-1)\cdots(M-N+1)}{M^N}\;.
\end{equation*}
When $M$ is large,
\begin{align}
\log{p_\text{inj}}&\sim M\int_0^f\log{(1-z)} dz\nonumber\\
&=-M\Big(f+(1-f)\log{(1-f)}\Big)\;,\label{eq:pinjint}
\end{align}
where $f=N/M$. Since
\begin{equation*}
f=2^{-(m-n)}\to 0
\end{equation*}
for large $m$ and $m>2n$, we can keep just the leading order term of \eqref{eq:pinjint} for small $f$ with the result
\begin{equation*}
p_\text{inj}\sim e^{-N^2/(2M)}=e^{-2^{2n-m-1}}\;.
\end{equation*}
For fixed $m/n>2$ this approaches 1 exponentially with $m$.

\subsection{Autoencoder complexity}

Since the autoencoder $\mathcal{A}$ is the composition of the decoder and encoder, $\mathcal{A}\in \mathcal{B}(m,m,k_\mathcal{A})$ where $k_\mathcal{A}=k_\mathcal{D}+k_\mathcal{E}$. However, one should beware that its complexity might be lower. If $\Pi:\{0,1\}^n\to \{0,1\}^n$ is any bijection, and
\begin{subequations}\label{eq:perm}
\begin{alignat}{2}
\mathcal{D}'&=\Pi\circ\mathcal{D}&&\in\mathcal{B}(m,n,k'_\mathcal{D})\\
\mathcal{E}'&=\mathcal{E}\circ \Pi^{-1}& &\in\mathcal{B}(n,m,k'_\mathcal{E})\;,
\end{alignat}
\end{subequations}
then $\mathcal{E}'\circ \mathcal{D}'=\mathcal{E}\circ \mathcal{D}$ defines the same autoencoder. The minimum complexity autoencoder has log-complexity
\begin{equation*}
\min_{\Pi} (k'_\mathcal{D}+k'_\mathcal{E})\;.
\end{equation*}
Because $\Pi$ only permutes the samples to the encoder in the definition of the acceptance, $\alpha$ is unchanged by this action.

Henceforth, when we write $\mathcal{A}\in \mathcal{B}(m,m,k_\mathcal{A})$, it is understood that $\mathcal{A}$ has internal dimension $n$ and $k_\mathcal{A}$ is the sum of the log-complexities of the decoder and encoder. 

\subsection{Information sufficiency}\label{sec:infosuff}

The phrase ``can be learned'' in property 3 needs elaboration in a formal definition. Rephrasing this as ``can be learned in principle'' and using information theory we can get a handle on this property.

We start by defining, for any $\mathcal{A}\in\mathcal{B}(m,m,k_\mathcal{A})$ and $x\in\{0,1\}^m$, 
\begin{equation*}
\delta(\mathcal{A},x)=\left\{
\begin{array}{ll}
1, & \text{if } \mathcal{A}(x)=x,\\
0, &\text{otherwise .}
\end{array}
\right.
\end{equation*}
Then
\begin{equation*}
N(\mathcal{A})=\sum_{x\in\{0,1\}^m}\delta(\mathcal{A},x)
\end{equation*}
is the number of elements fixed by autoencoder $\mathcal{A}$, and
\begin{equation*}
M(x)=\sum_{\mathcal{A}\in\mathcal{B}(m,m,k_\mathcal{A})}\delta(\mathcal{A},x)
\end{equation*}
is the number of autoencoders (in our function class) that fix element $x$. Assuming the function class $\mathcal{B}(m,m,k_\mathcal{A})$ is invariant under permutations and negations of the $m$ inputs and outputs, the number $M(x)$ is independent of $x$ so we can simply replace $M(x)$ by $M$.

It would be nice if $N(\mathcal{A})$ had the analogous property of being independent of $\mathcal{A}\in \mathcal{B}(m,m,k_\mathcal{A})$. This is clearly not true for the circuit-like function classes we will be using. The next best thing is to postulate the \textit{concentration-of-measure} (COM) property
\begin{equation}\label{eq:COM}
N(\mathcal{A})\sim|X|,
\end{equation}
where randomly selected autoencoders in the class $\mathcal{B}(m,m,k_\mathcal{A})$ all fix a set of the same asymptotic size, written as $|X|$, as the parameters $m$, $n$, and $k_\mathcal{A}$ get large. The symbol $X$ coincides with our symbol for a particular machine-learnable set, but in the defining property it can be any machine-learnable set for $\mathcal{A}\in\mathcal{B}(m,m,k_\mathcal{A})$ since they all have the same asymptotic cardinality. The $\mathcal{A}$-independence statement \eqref{eq:COM} defines what we will call the COM model.

In the COM model we have an explicit formula for the probability $p(\mathcal{A},x)$ that a uniformly sampled $\mathcal{A}\in\mathcal{B}(m,m,k_\mathcal{A})$ fixes a uniformly sampled $x\in\{0,1\}^m$ :
\begin{align*}
p(\mathcal{A},x)&=p(x|\mathcal{A})\,p(\mathcal{A})\\
&=\left(\frac{1}{|X|}\delta(\mathcal{A},x)\right)\frac{1}{2^{k_\mathcal{A}}}\;.
\end{align*}
Here is the corresponding mutual information:
\begin{align}\label{eq:MI}
I_{\mathcal{A}\leftrightarrow x}&=\sum_{\mathcal{A}\in\mathcal{B}(m,m,k_\mathcal{A})}\;\sum_{x\in\{0,1\}^m}p(\mathcal{A},x)\log_2\left(\frac{p(\mathcal{A},x)}{p(\mathcal{A})p(x)}\right)\nonumber\\
&=\log_2\left(\frac{1}{|X|\;2^{k_\mathcal{A}}}\right)-\log_2\left(2^{-k_\mathcal{A}}\;2^{-m}\right)\nonumber\\
&=m-\log_2|X|\nonumber\\
&=m-n+\log_2(1/\alpha)\;.
\end{align}
The derivation assumed the encoder is injective and used the fact that all terms in the sum of $\delta(\mathcal{A},x)\log_2 \delta(\mathcal{A},x)$ are zero because $\delta(\mathcal{A},x)\in\{0,1\}$.

When the mutual information formula is used for independent samples $x\in X^*$ known to have the property that they are fixed by one particular $\mathcal{A}$, their aggregate information may enable us to reconstruct that $\mathcal{A}$. In the COM model we have a quantitative condition for information sufficiency:

\begin{proposition}\label{prop:X*size}
When the learning algorithm is working under the assumption that every element of the example subset $X^*\subset X$ is fixed by a particular $\mathcal{A}\in\mathcal{B}(m,m,k_\mathcal{A})$ whose encoder is injective, then in the COM model there is sufficient information to reconstruct $\mathcal{A}$ provided
\begin{equation}\label{eq:X*size}
|X^*|>\frac{k_\mathcal{A}}{m-n+\log_2{(1/\alpha)}}\;.
\end{equation}
\end{proposition}
\begin{proof}
In the COM model we can use \eqref{eq:MI} for the information provided by each element $x\in X^*$. Since we need $k_\mathcal{A}$ bits to reconstruct $\mathcal{A}$, the condition for information sufficiency is
\begin{equation*}
|X^*|\;I_{\mathcal{A}\leftrightarrow x}>k_\mathcal{A},
\end{equation*}
from which criterion \eqref{eq:X*size} follows.
\end{proof}

Inequality \eqref{eq:X*size} is uninteresting without some context. Proposition \ref{prop:X*size} is silent on the \textit{algorithm} being used to learn $X$ from the small set of examples, and how this algorithm is restricted to autoencoders of a given complexity. The experimental half of our work addresses this and the results provide bounds on the information-theoretic threshold.

It is easy to experimentally establish a number $N_0$ such that if $|X^*|<N_0$, then an autoencoder that learns how to fix $X^*$ will fail to fix the rest of $X$. Such experiments provide a lower bound on the complexity $k_\mathcal{A}$ of the class of Boolean functions that the learning algorithm is operating under. Improving the bound, experimentally by increasing the $N_0$ target, becomes increasingly hard and eventually can no longer be considered easy. 

It is also easy to establish, again experimentally, a number $N_1$ such that if $|X^*|>N_1$, then an autoencoder can easily learn how to fix all of $X$. As with $N_0$, the work required to decrease $N_1$ eventually also becomes prohibitive. So while the numbers $N_0$ and $N_1$ only provide bounds on the true threshold \eqref{eq:X*size}, in practice they represent transitions because we are mostly interested in evidence that is easily obtained. In experiments on two kinds of machine-learnable sets we find that it is relatively easy to obtain $N_1$ and $N_0$ where the former exceeds the latter by only a small factor.

An ideal experiment would allow us to explicitly specify the parameter $k_\mathcal{A}$ of the learning algorithm. Though the algorithm described in Section \ref{sec:BTF} falls short of that, we believe that our estimates of $k_\mathcal{A}$ for that algorithm are in the right ballpark.   

\subsection{Evolution of machine-learnability}\label{sec:evolution}

Let $X_0, X_1, X_2,\ldots$ be a sequence of \textit{imperfect} machine-learnable sets. Imperfection is to be understood as follows. The autoencoders being used to learn these sets, all with internal dimension $n$ and some fixed bound on their complexity, are insufficient (lack capacity) to fix all the elements of their set. The best that the learning algorithm can do is find good minimizers (autoencoders) $\mathcal{A}_i$ of the average Hamming distance:
\begin{equation}\label{eq:Hamdist}
d_\text{ave}=\frac{1}{|X_i|}\sum_{x\in X_i}d_H\big(x,\mathcal{A}_i(x)\big)\;.
\end{equation}
As a result, $\mathcal{A}_i(X_i)$ is not the same as $X_i$ but a slightly different set, and motivates us to define the iteration
\begin{equation}\label{eq:setevolve}
\mathcal{A}_i(X_i)=X_{i+1}\;.
\end{equation}

To see why iteration \eqref{eq:setevolve} might converge, consider a set $X$ with the property that one iteration already produces a fixed point, that is
\begin{equation*}
\mathcal{A}(X)=X'\ne X\;,
\end{equation*}
but \begin{equation*}
\mathcal{A}'(X')=X'\;.
\end{equation*}
This can happen even when $\mathcal{A}'=\mathcal{A}$ (the autoencoder did not have to change to be successful) and is the scenario rendered in Figure \ref{fig:auto}. The middle layer shows the space of decoder-outputs and encoder-inputs, $\{0,1\}^3$, only six elements of which, $Y$, have a role in this autoencoder for a 6-element set $X$. The action of the decoder is best understood in terms of its inverse images, which decompose the input space of the autoencoder, $\{0,1\}^m$, into a disjoint union. The triangles in the diagram show how the decoder partitions this space. The points $\{0,1\}^m$ are rendered as a gray continuum, emphasizing that they vastly outnumber the points $\{0,1\}^3$. In this example six of the partitions are occupied by a single point of $X$. The decoder maps these to $Y$, and the encoder then to $X'$, shown embedded in another copy of $\{0,1\}^m$ at the top of the diagram. We get the single-iteration fixed-point property just by having the points $X'$ fall inside the same partitions used by the original points $X$. The evolution $X\to X'$ is indicated by the arrows in the lower copy of $\{0,1\}^m$. These map the original points to the points the encoder would like them to be (rendered white) while staying inside the same partitions.

\begin{figure}
    \centering
    \includegraphics[width=\columnwidth]{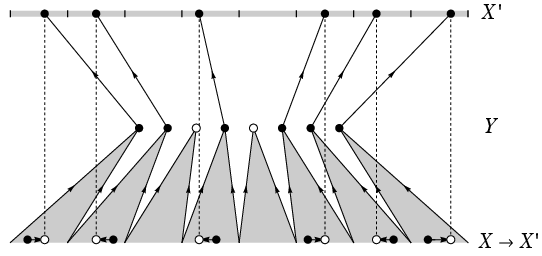}
    
    \caption{A 6-element set $X$ (black points in lower row) are mapped by the decoder to the set $Y$ (black points in middle row) and then to the set $X'$ (black points in upper row) by the encoder. The internal dimension is $n=3$ and $\alpha=6/8$. The data space $\{0,1\}^m$ has so many points it is rendered as a gray continuum, which is partitioned into the inverse images of the decoder (gray triangles). In this example, because the elements of $X'$ lie in the same partitions as the elements of $X$, the refinement $X\to X'$ yields a true learnable set.}
    \label{fig:auto}
\end{figure}

A more serious way for a set $X$ to fail the ``autoencoder test'' (than depicted in Figure \ref{fig:auto}), is when some of the decoder's partitions of $\{0,1\}^m$ contain more than one element of $X$. It is not correct to expect that in this case the iteration \eqref{eq:setevolve} collapses the joint occupants into a single element, thereby decreasing the size of $X$. This is a reasonable expectation only if we assume the learning algorithm initializes autoencoder $\mathcal{A}_{i+1}$ as $\mathcal{A}_{i}$, so that initially the decoder's partitions are the same. Better learning algorithms are not initialized this way, but given the opportunity to learn a better decoder, one where there are fewer instances of multiple occupancy in the partition of $\{0,1\}^m$. For example, in the algorithm described in the next section and used in our experiments, the parameters of the autoencoder are initialized randomly each time it is run. It is remarkable that even under these \textit{de novo} conditions the experiments show a systematic improvement in learning when the algorithm is presented with data as refined by \eqref{eq:setevolve}. The implied heuristic is that for $X_{i+1}$, unlike $X_{i}$, we know there exists an autoencoder ($\mathcal{A}_i$) satisfying the complexity bound that can perfectly learn at least some large fraction of the data (those which lie in singly-occupied decoder-partitions and are mapped by $\mathcal{A}_i$ into the same partition).

Up to now we have only considered what makes machine-learnable sets learnable---the existence of an autoencoder with certain properties---and not the learning process itself. Without a practical learning algorithm to construct the autoencoder, the definition of these sets would not have much value. This is addressed in the next section.

\section{Learning with boolnets}\label{sec:BTF}

Boolnets are networks of Boolean threshold functions (BTFs). The latter are Boolean functions where the 0/1 values on the hypercube are linearly separable:   
\begin{equation*}
f_w(x)=\text{sgn}(w\cdot x)\;.
\end{equation*}
The components of the continuous parameter vector $w$ are the BTF's \textit{weights}. 
We have switched to $\pm 1$ values for the Boolean variables ($-1$ corresponds to $0$),
\begin{equation*}
f_w:\{-1,1\}^m\to \{-1,1\}\;,
\end{equation*}
where $m$ is the input dimension. In a boolnet each BTF input is the output of another BTF or an input to the network itself. Each of the network's outputs is the output of some BTF in the network.

Because $x$ is discrete, there is a continuum of $w'$ that define the same BTF as $w$, and we write $w'\sim w$. Often some components of $w$ are so small in magnitude, relative to the others, that they have no effect on the value of $f_w$. The number of ``relevant'' inputs is the BTF's \textit{support}, written $\text{supp}(w)$. Another important characteristic of a BTF is its \textit{margin}:
\begin{equation}\label{eq:margindef}
\mu(w)=\min_{x\in \{-1,1\}^m}|w\cdot x|\;.
\end{equation}
By fixing the norm of the weight vector,
\begin{equation*}
\|w\|^2=m\;,
\end{equation*}
the complexity of a BTF can be controlled by bounding its margin. In \cite{elser2026learning} it is shown that the constraint
\begin{equation}\label{eq:margincon}
\mu(w)\ge \sqrt{\frac{m}{\sigma}}
\end{equation}
can only be satisfied if
\begin{equation}\label{eq:suppbnd}
\text{supp}(w)\le \sigma\;.
\end{equation}
The number $\sigma$ is called the support parameter. Bound \eqref{eq:margincon} is saturated when $\sigma=n$ is an odd integer and the weight vector has $n$ equal-magnitude components. 

A boolnet with BTF parameter $\sigma=3$ and sufficient size (network depth and width) is able to express arbitrary Boolean functions. That's because BTFs with three equal-magnitude weights implement the \textsc{Maj} (majority) gate. When one of the three relevant inputs is held fixed, say at $-1$, the BTF implements \textsc{And} or \textsc{Or}, depending on the sign of the weight of the fixed input. The setting $\sigma=3$ also admits BTFs with just one relevant input. These are \textsc{Copy} or \textsc{Not} gates, depending on the sign of the corresponding weight.

The margin constraint in boolnets serves multiple purposes. Positive margins ensure that the learned weights never lead to borderline situations, where the argument $w\cdot x$ of $\text{sgn}(\;)$ is small. Boolnets whose BTFs all have small support are sparse and interpretable (circuits comprising few-input gates). Sparsity is conferred through the ``neurons,'' instead of a globally imposed limit on the number of ``wires.'' In this work, where boolnets define the Boolean function class that the learning algorithm works with, the margin constraint gives us a handle on the complexity parameter $k_\mathcal{A}$ of the autoencoder.

\subsection{Boolnet complexity}

The learning algorithm in our experiments  works with the function class $\mathcal{B}(m,m,k_\mathcal{A})$ defined by boolnets that have a layered architecture and a global support parameter $\sigma$ constraining the BTF margins. The boolnet auto\-encoders are comprised of $L$ layers, where at each of the $m_\ell$ nodes in layer $\ell$ there is a BTF that takes input from the $m_{\ell -1}$ BTFs in layer $\ell-1$ or, when $\ell=1$, the $m$ autoencoder inputs. The $m_L=m$ BTFs in the final layer produce the autoencoder's outputs. All the BTFs in the boolnet also take input from a fixed-value node of the network.

The learning algorithm searches for BTF weights that satisfy constraints in $|X^*|$ instantiations of the boolnet, whose inputs and outputs are constrained by an element of $X^*$. The weight vectors $w$ of BTFs in layer $\ell$ have squared norm $\|w\|^2=m_{\ell-1}$ and satisfy
\begin{equation}
|w\cdot x|\ge \sqrt{\frac{m_{\ell-1}}{\sigma}}
\end{equation}
for each of the $|X^*|$ input vectors $x$ that arise at the BTF in the course of satisfying all the constraints posed by $X^*$. This constraint is equivalent to \eqref{eq:margincon} only if all possible $2^n$ values arise on the BTF's support when $\text{supp}(w)=n$ for the weight vector $w$ in the solution of the constraint problem. We will call this the \textit{diversity hypothesis}.

We get a good upper bound on the complexity parameter $k_\mathcal{A}$ (for the function class under consideration) when the diversity hypothesis is true. We specialize for the case $\sigma=3$, the support parameter used in all of our experiments. All BTFs will then have either 1 or 3 relevant inputs, and because even the magnitudes are fixed (equal) in the last case, the BTF variety comes only from the signs of weights. A BTF in layer $\ell$ is therefore limited to the following number of forms:
\begin{equation*}
C_\ell=2^1\binom{m_{\ell-1}+1}{1}+2^3\binom{m_{\ell-1}+1}{3}\;.
\end{equation*}
The first term corresponds to \textsc{Copy} or \textsc{Not}, and incrementing $m_{\ell-1}$ by 1 takes into account the single fixed-value network input (used by \textsc{And}/\textsc{Or}).
For the $L$-layer network we then have the following upper bound on the number of Boolean functions that can be synthesized:
\begin{equation*}
\prod_{\ell=1}^L C_\ell^{m_\ell}\;.
\end{equation*}
This overcounts because the nodes in the hidden layers may be freely permuted without changing the function. Negating all the inputs to a BTF in a hidden layer and compensating by negating its output (by flipping the signs of weights) also does not change the function. A better bound is therefore
\begin{equation*}
C_\text{net}=C_L^{m_L}\prod_{\ell=1}^{L-1} \frac{C_\ell^{m_\ell}}{m_\ell !\;2^{m_\ell}}\;.
\end{equation*}
This still is only an upper bound. For example, when few BTFs have support 3---the rest are just copying/negating values from one layer to the next---these can be staged in multiple ways in a deep network without changing the function. We suspect that $C_\text{net}$ covers the dominant sources of overcounting in large networks.

Under the assumption that the diversity hypothesis is true, in a network with 
\begin{equation*}
\mathcal{N}=\sum_{\ell=1}^L m_\ell
\end{equation*}
BTFs and typical layer size $\overline{m}$, the number of bits needed to specify a Boolean function in our class has asymptotic growth
\begin{equation}\label{eq:knet}
k_\text{net}=\log_2 C_\text{net}\sim 2 \mathcal{N}\log_2 \overline{m}\;.
\end{equation}
While the empirical bounds on the true complexity $k_\mathcal{A}$ obtained in our experiments are in the same ballpark, we will see that $k_\text{net}$ is definitely an underestimate. The reason for this is also brought out in the experiments: the diversity hypothesis is strongly violated in the decoder-half of our autoencoders. The constraints on a BTF are weaker when only a proper subset of the $2^n$ possible inputs, on a $w$ with support $n$, are present during learning. As a result, more BTF types arise in practice and $C_\text{net}$ undercounts the true number.

\subsection{Learning algorithm}

We use the reflect-reflect-relax (RRR) constraint satisfaction algorithm to learn the weights in our boolnets. To use this algorithm the constraints are represented as geometrical sets $A$ and $B$, such that solutions to the constraint problem are points in the intersection, $A\cap B$. The sets $A$ and $B$ live in a continuous space $Z$ whose dimensions correspond to all the weights, inputs and outputs of the BTFs.

Associated with BTF $i$ is the set $A_i$ corresponding to all solutions of
\begin{align*}
\text{sgn}(w_i\cdot x_i)&=y_i\\
|w_i\cdot x_i|&\ge \sqrt{\frac{m_i}{\sigma}}\;,
\end{align*}
where $m_i$ is the number of inputs and $\sigma$ is the support parameter. Though the scalar variable $y_i$ should be identified with components of input vectors of other BTFs (or an output of the boolnet), $y_i$ is treated as an independent variable in constraint $A_i$. In this way the constraints at all the BTFs are decoupled or ``factorized'' in the $A$ constraint:
\begin{equation*}
A=A_1\times A_2\times \cdots \times A_\mathcal{N}\subset Z\;.
\end{equation*}
The rationale for doing this is that finding the distance-minimizing change $(w,x,y)\to(w',x',y')$ that satisfies $A_i$ is easy to compute when the BTFs are decoupled. Distance-minimizing transformations that satisfy constraints are called projections, and the RRR algorithm uses two: $P_A$ and $P_B$. The equations that define set $B$ are again decoupled: $B_i$ is the simple equality that applies between the output $y_i$ and the $x$-components of the receiving BTFs. $B_i$ also constrains $w_i$ to have squared norm $m_i$. Putting the norm constraint in $B$ makes the projection to $A$ (which also involves the weights) much simpler.

\begin{figure*}
    \centering
    \includegraphics[width=2\columnwidth]{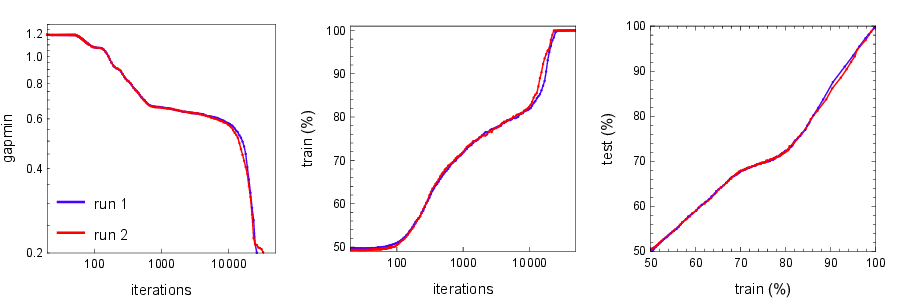}
    
    \caption{Gap-plot (left), train-accuracy (center), and learning curve (right) in two runs of the RRR constraint-satisfaction algorithm when learning the $8\times 8$ Rorschach set. There are 200 points in each plot, logarithmically spaced in iteration count.}
    \label{fig:RS4runs}
\end{figure*}

The strategy of replicating variables, like BTF inputs and outputs, to make constraint projections easy to compute is called \textit{divide-and-concur} \cite{gravel2008divide}. If inputs and outputs were only subject to the $A$ constraint, then only the outputs $y$ would be discrete ($\pm 1$). But the equalities between inputs and outputs in the $B$ constraint, also called ``concur'', ensure that in a solution the components of $x$ will be discrete as well.

In gradient-descent learning algorithms a layered network architecture makes sense because the updates use the back-propagation formula. This is not a consideration with divide-and-concur. The decoupling of the constraint projections at the level of individual BTFs, in both $A$ and $B$, means that these computations can be performed concurrently even in non-layered networks.

In our learning algorithm the replication of variables is taken one step further. The $(w_i,x_i,y_i)$ trio at BTF $i$ is replicated as many times as there are training data, $|X^*|$. These ``data replicas'', differing only in the constraints imposed at the input and output nodes of the boolnet, are completely independent and subject to the $A$ and $B$ constraints described above. However, a key detail should not be overlooked: the data-replicas should share the same set of weights! This is taken care of in the $B$ constraint, where data-relicas of the weight vectors are projected to the same vector (the average of the replicas) in addition to a vector of the correct norm (by scaling).

RRR is usually initialized with a random point $z\in Z$. Recall that ``point'' means a $(w,x,y)$ trio for each BTF and each training-data item. Since all three variable types have mean-squared value 1 in a solution, and there are symmetries that allow the signs to be flipped, we initialize with uniform random samples between $1$ and $-1$ for all the components of $z$. The initial point is updated by the RRR iteration rule:
\begin{equation}\label{eq:RRR}
z\to z+ \beta\Big(P_B(R_A(z))-P_A(z)\Big)\;.
\end{equation}
The parameter $\beta>0$ is analogous to the learning rate of gradient-descent optimizers and $R_A(z)=2P_A(z)-z$ is ``reflection in constraint $A$.'' The name RRR is a reference to the fact that the iteration rule can be written as a relaxation (by parameter $\beta$) of $R_B\circ R_A$. The form \eqref{eq:RRR} is more transparent about what happens at a fixed point $z^*$ :
\begin{equation*}
P_B(R_A(z^*))=P_A(z^*)\in A\cap B\;.
\end{equation*}
Though this follows even without $R_A$, the reflection is essential for fixed points to be attractive. For the history of RRR and its application to a broad selection of nonconvex problems, see \cite{elser2025solving}. Additional details on using RRR to train boolnets can be found in \cite{elser2026learning}.

Our experimental results are mostly conveyed through three kinds of plots. To assess the progress of the constraint solver one is interested in the \textit{gap}, defined by
\begin{equation*}
\Delta=\|P_B(R_A(z))-P_A(z)\|\;.
\end{equation*}
This measures the currently achieved distance between the two constraint sets. By the iteration rule it is also the speed of $z$ as it moves through $Z$. We normalize $\Delta$ as a root-mean-square quantity, where the averaging is over BTFs and training-data items. More concretely, the gap in the plots represents the typical change (total distance) of the variables in one $(w,x,y)$ trio. Our ``gap-plots'' show \texttt{gapmin}, the minimum value of the gap achieved at the time of the current iteration. Our iterations-axes are logarithmic, so that (not surprisingly) the gap is nearly always strictly decreasing.

The other two plots show the progress of the autoencoder's accuracy. The latter is defined as the average fraction of the autoencoder's $m$ outputs that match its inputs, or using definition \eqref{eq:Hamdist}, $1-d_\text{ave}/m$. In the accuracy called \texttt{train} the averaging is over all elements of $X^*$, while the complement of $X^*$, in a set of $10^4$ samples of $X$ not seen by the constraint solver, is used to evaluate the accuracy called \texttt{test}. ``Train-accuracy'' shows the evolution of \texttt{train} with iterations. ``Learning-curves'' show the joint evolution of \texttt{train} and \texttt{test}.

Examples of the three kinds of plots are shown in Figure \ref{fig:RS4runs}. These show the learning of a Rorschach set, the subject of the next section. The plots show two runs (red and blue) and convey that in all respects---optimization, accuracy, generalization---the results are independent of the constraint solver's initialization. The gap axis (left panel) is logarithmic to better show the transition from a prolonged ``active-learning'' phase, where the gap decreases slowly (though systematically), to a more ``convergent-learning'' phase. From the train-accuracy (center) we see that training accuracy can be perfect while the gap is nonzero. This means that the optimizer is still working on satisfying all the margin constraints even though the Boolean function the boolnet has implemented is correct. The learning-curve (right) shows that the learning of $X$ has succeeded because the test accuracy is also at 100\%.

All of our experiments were performed with \texttt{boolearn}. Readers interested in reproducing or extending our results can find data sets at the \href{https://github.com/veitelser/boolearn}{\texttt{boolearn} site}. The learning algorithm, \texttt{train}, has two hyperparameters: the time-step $\beta$ in \eqref{eq:RRR}, and a metric relaxation rate $\gamma$ \cite{elser2026learning}. We used $\beta=0.3$ and $\gamma=10^{-3}$ in all the experiments.

\section{Rorschach sets}\label{sec:RS}

Rorschach patterns are a simple kind of learnable set. Figure \ref{fig:RSsamples} shows some $10\times 10$ patterns. Pixels related by the vertical mirror have the same contrast, or for added interest, the reversed contrast (lower row). The learning algorithm knows nothing about the spatial arrangement of the pixels---each data item is just a string of 100 bits. The true ``content'' in these data is the value of just 50 of the bits plus an additional bit on whether the values of the corresponding mirror-bits should be negated. The autoencoder that learns this data will have to decode the mysterious strings into their true content, and then be able to encode the content back to the strange code.

\begin{figure}[b]
    \centering
    \includegraphics[width=\columnwidth]{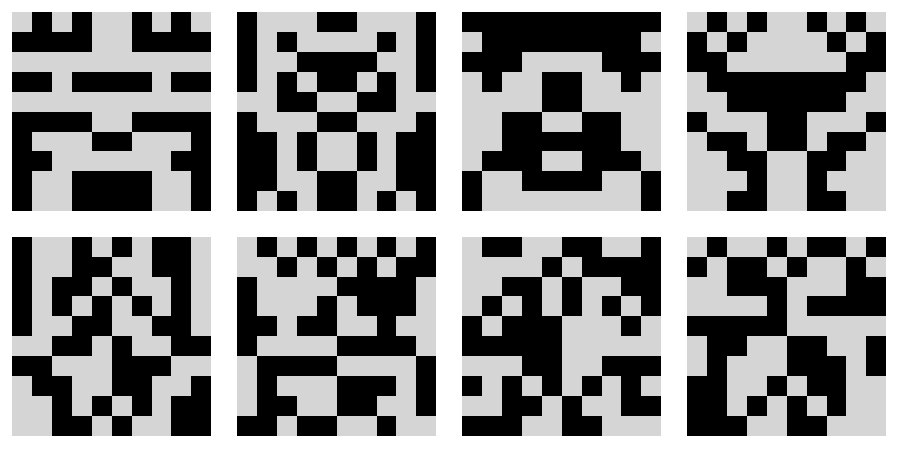}
    
    \caption{Samples of $10\times 10$ Rorschach patterns. In the samples of the lower row the mirrored pixels have had their contrast reversed.}
    \label{fig:RSsamples}
\end{figure}

\subsection{Scaling}

We first study how the learning of Rorschach sets scales with size. Our experiments cover sizes $2s\times 2s$, $2\le s\le5$. Through experimentation we arrived at the following 3-layered autoencoder architecture:
\begin{equation}\label{eq:RSarch}
4s^2\;\to \;2s^2+4\;\to\; 6 s^2\;\to\; 4s^2\;.
\end{equation}
It was a surprise to us that architectures this simple could learn these sets. We augmented the decoder output by three bits (over the $2s^2+1$ bits of content) to test the rejection-sampling scheme for generating Rorschach patterns. The size of the encoder's single hidden layer can be reduced, though we do not know by how much.

\begin{figure*}
    \centering
    \includegraphics[width=2\columnwidth]{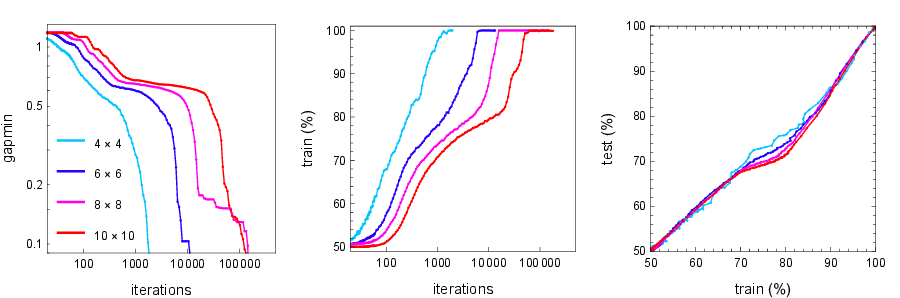}
    
    \caption{Scaling of Rorschach-set learning as the size of the patterns is increased.}
    \label{fig:RSscaling}
\end{figure*}

Using information sufficiency (Prop. \ref{prop:X*size}) to select the sizes of the training sets is probably not the best choice for investigating scaling. The work to learn a set surely diverges near the threshold, and the degree of divergence surely depends on $s$. By Proposition \ref{prop:X*size} and the estimate \eqref{eq:knet}, the threshold size $|X^*|$ grows linearly with the number of BTFs. We found that
\begin{equation}\label{eq:Xksize}
|X_s^*|=64 s^2
\end{equation}
examples, for the $2s\times 2s$ patterns, puts us safely on the information-sufficient side of the transition. The experiments extend to $s=5$. Another round of experiments will address the location of the learnability transition for $s=4$.

The left panel of Figure \ref{fig:RSscaling} shows the evolution of the gap-plots with $s$. Learning becomes increasingly protracted (small slope region) with increasing $s$, but not alarmingly so. The number of iterations before the plunge to the RRR fixed-point appears to grow quadratically with the number of BTFs, $144 s^4$. This behavior is mirrored in the train-accuracy (center panel). The accuracy reaches 100\% when the gap is below $0.2$. Learning curves are shown in the right panel. The test accuracies keep pace with the training accuracy, also reaching 100\%. We will see learning curves with more structure when the number of examples approaches the learnability threshold.

\subsection{Decoder}

The architecture \eqref{eq:RSarch} is too small for the ``obvious'' autoencoder circuit. This is where one representative of each mirror pair is threaded through the bottleneck and routed to the same position on the output side. The partner of one mirror pair is also held in reserve in the decoder's output, to determine if the mirror pairs across the whole pattern should have their values negated.

Let $x_{p}, x_{-p}$ be the Boolean values of the reserved mirror pair whose equality/inequality determines contrast reversal, and $x_{q}, x_{-q}$ any of the other mirror pairs. The decoder output only has space for one of the latter variables, say $x_{q}$. But the encoder can compute the other one with the formula
\begin{equation*}
x_{-q}=\text{\textsc{Xor}}(x_{q},x_{p},x_{-p})
\end{equation*}
and route this value to position $-q$ in the autoencoder's output. However, \textsc{Xor} is not linearly separable and must be computed with several BTFs. It can be done in a 2-layered encoder with four 3-input \textsc{And} gates per pixel $q$, whose outputs are then combined with \textsc{Or}. But that requires $4\times 2 s^2$ nodes in the hidden layer. Composing two 2-input \textsc{Xor} gates reduces the layer width but doubles the number of layers.

As inspiration for what the decoder does instead, consider the way mirror pairs are identified in correlation analysis. In ordinary Rorschach patterns (without the possibility of contrast reversal) pixel pairs $p$ and $-p$ are identified by the property
\begin{equation}\label{eq:RSave}
\langle x_{p}\; x_{-p}\rangle>0\;,
\end{equation}
where the angle brackets denote averaging over examples. Pixel pairs not mirror-related have zero average. When there is contrast reversal, correlation  \eqref{eq:RSave} is zero, but 4-variable combinations
\begin{equation}\label{eq:RS2ave}
\langle x_{p}\; x_{-p}\;x_{q}\;x_{-q}\rangle>0\;,
\end{equation}
now have positive average.

\begin{figure}[b]
    \centering
    \includegraphics[width=\columnwidth]{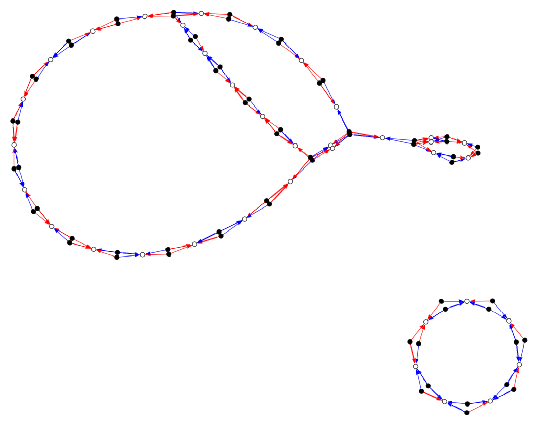}
    
    \caption{Decoder circuit, comprising two connected components, of an autoencoder that has learned the $8\times 8$ Rorschach set. The black node-pairs correspond to mirror-related pixels. White nodes are the decoder's outputs. The function of the circuit can be read off the diagram when BTFs with equal magnitude weights are placed at each white node and negative weights are placed at all the red arrows. The pattern of weights gives the output nodes a cyclic order.}
    \label{fig:RS4dec}
\end{figure}

Property \eqref{eq:RS2ave} is exactly the principle behind the RRR-trained decoder. The complete decoder circuit from an $8\times 8$ experiment is shown in Figure \ref{fig:RS4dec}. The decoder's $64$ inputs are rendered black, its $32+4$ outputs are the white nodes. The BTF at each output has only four inputs that are relevant when the latter have Rorschach symmetry. These are red/blue arrows, where red means the weight for that input is negative. Recall that the decoder is initialized as fully connected with uniform random weights. Figure \ref{fig:RS4dec} shows the decoder after learning, and when all but the four largest-magnitude weights at each BTF are set to zero---without changing the autoencoder's function.

The decoder circuits are not unique. The one in Figure \ref{fig:RS4dec} is typical, with a well defined cyclic order as brought out in the smaller of the two connected components. Decoders for the other Rorschach sizes have the same design principle. Input pairs, corresponding to mirror-pixel pairs, send like-sign-weighted values in a clockwise sense, and opposite weight-signed values counterclockwise. The value $y$ of any of the decoder's outputs is given by the formula
\begin{equation}\label{eq:RSBTF}
y=\text{sgn}(x_{p}+x_{-p}+x_{q}-x_{-q})\;,
\end{equation}
where $(p,-p)$ is the ``clockwise'' pixel pair and $(q,-q)$ the ``counterclockwise'' pair. The continuous weights found by RRR have been replaced with the simpler, equivalent set
\begin{equation*}
(w_{p},w_{-p},w_{q},w_{-q})\sim (1,1,1,-1)
\end{equation*}
when the pairs $(x_{p},x_{-p})$ and $(x_{q},x_{-q})$ have Rorschach symmetry. In particular, when $x_{p}=x_{-p}$ and $x_{q}=x_{-q}$, then $y$ inherits the value of $x_{p}$. If instead $x_{p}=-x_{-p}$ and $x_{q}=-x_{-q}$, then $y=x_{q}$.

We frankly do not understand the rationale behind the decoder design. One half of the Rorschach pattern is stored in the same number of bits, but in two ways---depending on contrast reversal---related by a cyclic shift. Does this make the job of the encoder easier? Or, as suggested by the correlation analysis perspective, does the structure of the circuit implicate an incremental, correlation-driven learning process? We hope to address such questions in future work. In any case, the design of the decoder circuit is an instance where \textit{we} have learned something new.

\subsection{Acceptance}

Once a good autoencoder is in hand, one can use it and its encoder-half to determine the acceptance using \eqref{eq:accept}. We will demonstrate for the $8\times 8$ Rorschach set. The distribution of Hamming distances
\begin{equation*}
d(y)=d_H\Big(\mathcal{E}(y),\mathcal{A}(\mathcal{E}(y))\Big)
\end{equation*}
for $10^4$ $y$-samples is shown in Figure \ref{fig:RS4hamdist}. We see that the acceptance of our autoencoder is about 19.2\%. We might have expected a number closer to 12.5\% ($2^{-3}$)  since the true entropy of the $8\times 8$ patterns is $32+1$ bits and our encoder has $32+4$ inputs. But that assumes the encoder is injective. As an extreme case, suppose the encoder is able to simply ignore three of its inputs. The acceptance, with this 8-to-1 breakdown of injectivity, would be 100\%.

\begin{figure}
    \centering
    \includegraphics[width=.9\columnwidth]{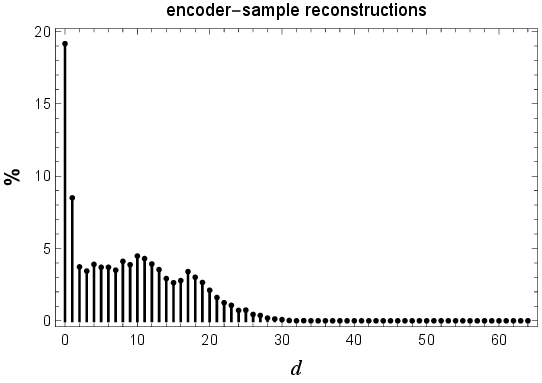}
    
    \caption{Distribution of Hamming distances $d$ when generating $8\times 8$ Rorschach patterns with an autoencoder. The fraction of patterns with $d=0$ is the autoencoder's acceptance.}
    \label{fig:RS4hamdist}
\end{figure}

\subsection{Learnability transition}
\label{sec:RS-learnability-transition}
The learned decoder design (Fig. \ref{fig:RS4dec}) shows that the learning algorithm wholeheartedly eschews the diversity hypothesis on which the small-support-through-large-BTF-margins mechanism is based. All BTFs in the decoder receive input from two pairs of mirror-related pixels on which the diversity hypothesis is strongly violated. When the weights in BTF \eqref{eq:RSBTF} are scaled to have squared norm $m$ ($m=64+1$ in the fully connected architecture), the BTF margin equals $2\times\sqrt{m/4}=\sqrt{m}$. Margins of this size are allowed with support parameters $\sigma\ge 1$ and $\sigma=3$ was used in the experiments.

The diversity hypothesis is upheld in the overwhelming majority of small input-sets, even in the case of Rorschach data. On the other hand, since the learning algorithm may end up taking input mostly from the minority where it is violated, we should suspect results that depend on the hypothesis, such as the complexity estimate \eqref{eq:knet}. Using this estimate on the architecture \eqref{eq:RSarch} for $s=4$, we obtain $k_\text{net}=3645$ bits. Taking this value for $k_\mathcal{A}$ in \eqref{eq:X*size}, we arrive at $|X^*|=120$ for the number of examples at the transition. This is smaller by a factor of 8.5 than the number \eqref{eq:Xksize} used in the scaling study.

The object of the next round of experiments is to bracket the  $8\times 8$ Rorschach transition, for boolnets with architecture \eqref{eq:RSarch} and $\sigma=3$ constraining the margins. The only difference from the $s=4$ scaling experiment is varying the number of examples $X^*$ used for training. Results are summarized in Figure \ref{fig:RS4trans}, where $|X^*|$ ranges between 100 and 500.

\begin{figure*}
    \centering
    \includegraphics[width=2\columnwidth]{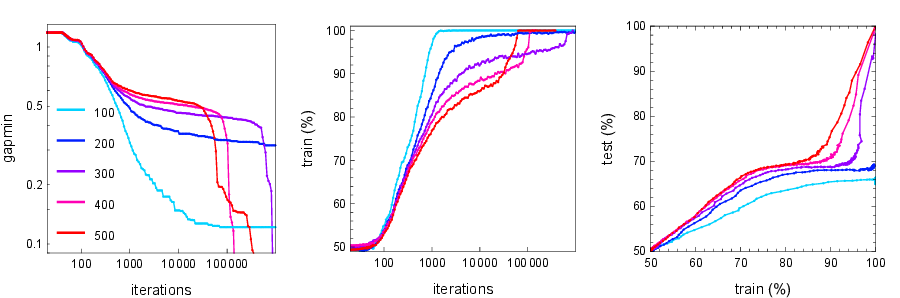}
    
    \caption{Learning the $8\times 8$ Rorschach set with different numbers of examples.}
    \label{fig:RS4trans}
\end{figure*}

Taken together, the three plots provide strong evidence that the learnability transition is located between 100 and 200 examples. There are ``aha'' events at iteration counts where the gap (left panel) takes an abrupt plunge and the training accuracy (center panel) surges upward. As the number of examples is reduced, the iterations required to arrive at the aha-event increases sharply. The experiments were terminated at $10^6$ iterations and bear witness to consummated learning only for 300 examples and higher. But up to the iteration cutoff, the 200-examples plots show the same behavior as the plots for 300 examples. On the other hand, the learning curve (right panel) for 100 examples  is qualitatively different, which we interpret as indicating information insufficiency.

Knowing the location of the learnability transition has little practical value. As demonstrated by the Rorschach sets, to avoid an iterations-blowup one only needs to increase the number of examples at the transition by a small factor. Studies of the learnability transition mostly offer insights about the learning algorithm, e.g. how does the work diverge as the transition is approached? One needs to bear in mind that algorithms can be compared only if they are able to work with the same Boolean function classes, like margin-constrained boolnets.

\section{Wild machine-learnable sets}\label{sec:wild}

Rorschach sets are very special and invite the critique that the experimental results are not representative of machine-learnable sets in general. The aim of this section is to address these concerns.

\subsection{Random-encoder sets}\label{sec:randencoder}

We first consider proto-machine-learnable sets $X_0$ constructed as follows. A  boolnet with architecture
\begin{equation}\label{eq:wildenc}
32\;\to\;64\;\to\;64\;\to\;64\;\to\;64
\end{equation}
and random weights was created to be the encoder $\mathcal{E}_0$ of the set's autoencoder $\mathcal{A}_0$ whose decoder is unknown. The idea is to uniformly sample the encoder's inputs to create samples $X_0^*$ from which one can try to learn $\mathcal{A}_0$. Because bijections such as \eqref{eq:perm} change the decoder and encoder, the learned encoder (when working in tandem with the learned decoder) may even have an architecture simpler than \eqref{eq:wildenc}. To aggressively test this possibility, we used
\begin{equation}\label{eq:wildauto}
64\;\to\;48\;\to\;32\;\to\;48\;\to\;64
\end{equation}
for the autoencoder architecture. But bijections have limited scope and it is highly improbable that a boolnet with architecture \eqref{eq:wildauto} can learn the random encoder-generated set $X_0$ exactly. Accordingly, $X_0$ will be the starting point of machine-learnable-set evolution as described in Section \ref{sec:evolution}. All the autoencoders $\mathcal{A}_0, \mathcal{A}_1,\ldots$ of the evolving sets $X_0, X_1, \ldots$ will have architecture \eqref{eq:wildauto}.

The boolnet of the random encoder $\mathcal{E}_0$ was constructed as follows. Each BTF takes input from either one or three nodes (BTF outputs or network inputs) in the layer below. The two cases have equal probability, and the selection of input nodes is uniformly random, with one condition: every node (in the layer below) finds itself in the input of some BTF. All the weights have equal magnitude and equally-probable signs. In circuit terms, $\mathcal{E}_0$ is comprised of \textsc{Copy}, \textsc{Not}, and 3-input-\textsc{Maj} between adjacent layers such that there are no ``dead-ends''. Two random encoders were created in this way, differing only in their wiring and negations. Each defined a proto set $X_0$ through its examples and was evolved as in Section \ref{sec:evolution}.

\begin{figure*}
    \centering
    \includegraphics[width=2\columnwidth]{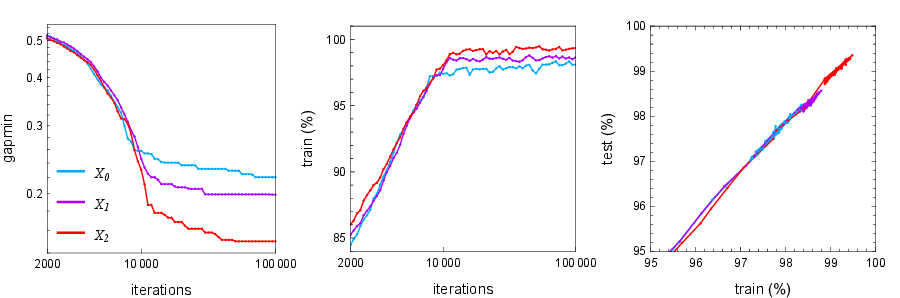}
    
    \caption{Learning the proto-set $X_0$ created by a random encoder and two rounds of its evolution ($X_1$ and $X_2$) under rule \eqref{eq:setevolve}.}
    \label{fig:wildevo}
\end{figure*}

\begin{figure}[b]
    \centering
    \includegraphics[width=.9\columnwidth]{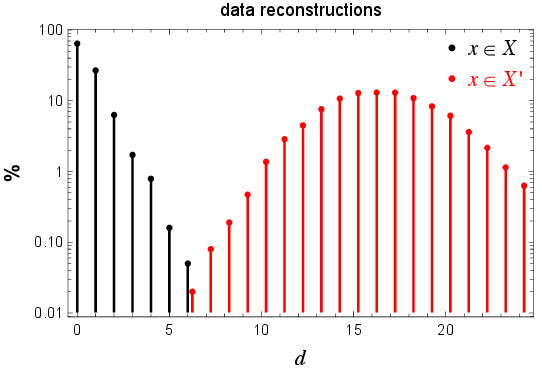}
    
    \caption{Distribution of Hamming distances $d_H(x,\mathcal{A}_2(x))$ for the autoencoder $\mathcal{A}_2$ that learned the twice-evolved wild set $X_2$. The plot compares samples $x\in X_2$ with samples $x\in X'_2$, where $X'_2$ is another, randomly constructed and twice-evolved wild set.}
    \label{fig:wildreconham}
\end{figure}

Figure \ref{fig:wildevo} compares the gap, train-accuracy, and generalization for the proto-set $X_0$ and the first two rounds of evolution, $X_1$ and $X_2$. The learning algorithm was trained with 2000-element subsets of these sets. Because the gaps saturate at nonzero values and the accuracies never quite reach 100\%, these are all imperfect learnable sets. But the evolution is as predicted, with accuracies improving and learning becoming slightly easier (fewer iterations). The detail of the learning curve (right panel) shows that both accuracies are above 99\% for $X_2$.

From the accuracies (train or test) in the plots, defined as the probability of a correctly reconstructed input-bit, we have no way of knowing if there is a small minority of poorly reconstructed samples, or whether the input-output Hamming distance is uniformly small over all the samples. This is resolved in Figure \ref{fig:wildreconham}, which shows the distribution
\begin{equation*}
d_H(x,\mathcal{A}_2(x))
\end{equation*}
where $\mathcal{A}_2$ is the autoencoder that the algorithm learned from the twice-evolved wild set $X_2$ and $x\in X_2$. Also shown (in red) is the distribution of distances when $x\in X'_2$, where $X'_2$ was twice-evolved from a different, random proto-set $X'_0$. We see that imperfection manifests itself as uniformly small Hamming distances $d\le 5$ (black points). Moreover, there are \textit{no} small distances (red points) when autoencoder $\mathcal{A}_2$ is tasked with recognizing elements of the other wild set. ``It's Greek to me,'' is how autoencoder $\mathcal{A}_2$ would describe set $X'_2$.

The zero-Hamming-distance definition \eqref{eq:accept} of the acceptance is too strict when a learnable set is imperfect. 
For the latter we use
\begin{equation}
\alpha=\underset{y\in\{0,1\}^n}{\text{prob}}\Big(d_H\big(\mathcal{E}(y),\mathcal{A}(\mathcal{E}(y))\big)\le d_0\Big)\;,
\end{equation}
where $d_0$ is the distance allowance needed by the autoencoder to recognize elements of the imperfect set it has learned---through no fault of its own---imperfectly. From Figure \ref{fig:wildreconham} we see that $d_0=5$ for set $X_2$. Using this in the definition of acceptance, we find that autoencoder $\mathcal{A}_2$ has acceptance 99.8\%. Since the samples of the starting set $X_0$ were not subject to rejection-sampling by an autoencoder (because $\mathcal{A}_0$ still had to be learned), its acceptance was effectively 100\%. The high acceptance of $\mathcal{A}_2$ is evidence that set-evolution preserves this property.

\subsection{Down-sampled MNIST}

We originally dismissed MNIST as a target of machine learnability because the diversity exhibited by the handwritten digits is continuous in nature, not combinatorial. However, through Fourier down-sampling, the continuous content can be reduced. The left panel of Figure \ref{fig:booLNIST} shows 64 samples of ``booLNIST'', created by down-sampling 5\,000 of the original $28\times 28$ images to $12\times 12$, binarizing, and cropping the result to $8\times 8$. 

\begin{figure*}
    \centering
    \includegraphics[width=2\columnwidth]{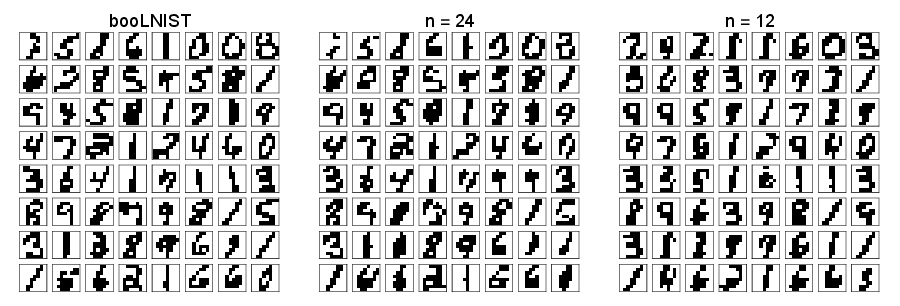}
    
    \caption{Samples of booLNIST images (left) and their evolution by autoencoders with waists $n=24$ and $n=12$.}
    \label{fig:booLNIST}
\end{figure*}

Though booLNIST looks combinatorial, we are left with a new challenge: How do we proceed when the entropy of the learnable set is unknown? A moderate overestimate of the entropy defines the size $n$ of the autoencoder's waist (in bits). For Rorschach sets we knew $n$ from symmetry, and for random-encoder sets we used injectivity to identify $n$ with the number of encoder inputs. For booLNIST we need measurements or experiments to estimate $n$. 

Estimating the entropy of machine-learnable sets is a tricky business. The number of elements we are given for learning the set is so pathetically small that one has to marginalize over all but some fraction of the components of the elements of $X^*$ to get a measurable probability distribution (so the counts of a significant number of elements of the marginalized set are greater than 1). By optimizing over the selection of components one can obtain bounds on the entropy. For booLNIST, marginalizing on all but an optimized set of 40 pixels we get a lower bound of 18 bits. Not surprisingly, these lower-bound pixels are near the image edges. The same procedure to get an upper bound, now with 20 pixels, yields an entropy of 38 bits and pixels near the image center. It's not possible to improve these bounds for the simple reason that even the complete data set of 5\,000 elements is too small for meaningful statistics.

The machine-learnable-set framework provides a practical solution to the entropy estimation problem. The idea is to first learn the whole set, and then evaluate the entropy as $n-\log_2(1/\alpha)$. Here $n$ is the width of the auto\-encoder's waist and $\alpha$ is its acceptance. Entropies defined in this way come with the understanding that the set was learned with a particular autoencoder architecture. We demonstrate with booLNIST and architectures
\begin{equation*}
64\;\to\; 32\;\to\; n\;\to\; 64\;\to\;64\;\to\; 64\;,
\end{equation*}
and try two values for the waist $n$ : $12$ and $24$. Clearly, only imperfect learning is possible with these simple architectures, in particular, when $n$ is small. But we can use the deterministic refinement \eqref{eq:setevolve} to evolve booLNIST into a true learnable set for each $n$. This seems like a bold plan, but in the present setting it is a very natural thing to do.

\begin{figure*}
    \centering
    \includegraphics[width=2\columnwidth]{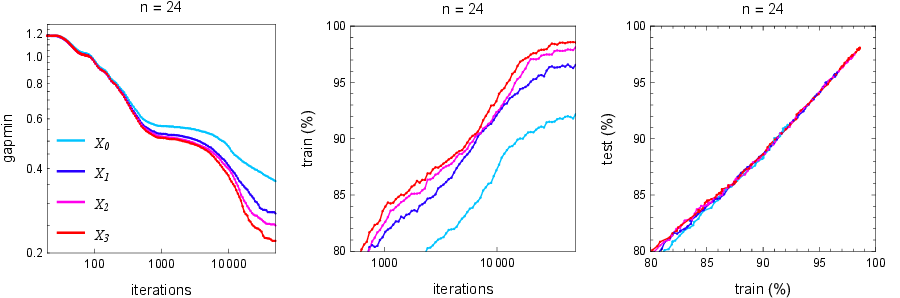}
    
    \caption{Evolution of the gap (left) and train-accuracy (center) as booLNIST undergoes three rounds of evolution with the $n=24$ autoencoder. Evolution has no effect on the train-test accuracy relationship (right panel).}
    \label{fig:booLNIST24}
\end{figure*}

\begin{figure*}
    \centering
    \includegraphics[width=2\columnwidth]{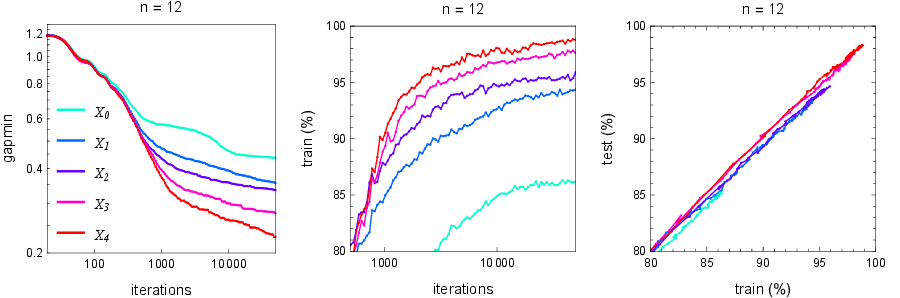}
    
    \caption{Same as Figure \ref{fig:booLNIST24} but for the $n=12$ architecture and an additional round of evolution.}
    \label{fig:booLNIST12}
\end{figure*}

\begin{figure*}
    \centering
    \includegraphics[width=2\columnwidth]{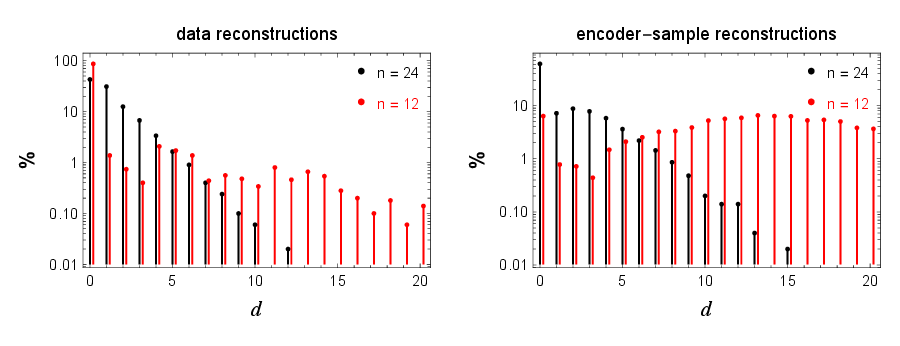}
    
    \caption{Distribution of Hamming distance when the two booLNIST autoencoders reconstruct data (left) and try to reconstruct encoder-generated data (right).}
    \label{fig:booLNISTham}
\end{figure*}

Consider two values for the waist, $n_1<n_2$, and let $\overline{X}_{1}$ and $\overline{X}_{2}$ be the corresponding evolved sets after many iterations of \eqref{eq:setevolve}. Both evolutions start with the same set $X^*$. Now if $x\in X^*$ evolves to $x_1\in \overline{X}_{1}$ and $x_2\in \overline{X}_{2}$ under the two autoencoder architectures, then we expect
\begin{equation*}
d_H(x,x_1)\ge d_H(x,x_2)\;.
\end{equation*}
Elements of the set learned under the architecture with the narrower waist have to change more to be fixed by their autoencoder. Knowing this, why would one learn with the narrower waist? The reason is simple: $|X_1|\ll|X_2|$, because $n_1<n_2$. By accepting larger refinements of its elements, the set $X_1$ has managed to do a better job at distilling the content of the original set-excerpt $X^*$. 
In the case of booLNIST it means that images $x_1$ generated by rejection-sampling with the evolved autoencoder $\overline{\mathcal{A}}_1$ of waist $n_1$ will ``look more like numbers'' than images $x_2$ generated with $\overline{\mathcal{A}}_2$ (and its larger waist). This testable prediction (and curiosity) motivated the booLNIST experiments. We made a vow to report our findings regardless of outcome.

All experiments used 2\,000 training data. Because the test-accuracy (based on the remaining 3\,000 data) always closely tracked the train-accuracy, increasing the number of training data would not have changed the results.

Figure \ref{fig:booLNIST24} shows the effect of three rounds of set-evolution for the $n=24$ architecture. The curve marked $X_0$ corresponds to the original booLNIST data. The gap reaches smaller values and in fewer iterations as the set evolves, and the accuracies improve as well. Set $X_3$ reaches accuracies $(98.5\%,98.0\%)$ at the end of training (50\,000 iterations). Results for the $n=12$ architecture are shown in Figure \ref{fig:booLNIST12}. Now the accuracies for the original booLNIST data are quite low, $(86.2\%,85.2\%)$, but reach $(98.8\%, 98.3\%)$ after four rounds of evolution. With the smaller waist the test-accuracy (right panel) shows a slight but systematic improvement with set-evolution.

One motivation for this experiment with image data was to directly observe the evolution of a set as its learnability improves. To see this, we return to Figure \ref{fig:booLNIST}. The result of three rounds of evolution with $n=24$ is shown in the center panel, and $n=12$ after four rounds is on the right. We need to remember that the learning algorithm does \textit{not} tweak individual images by a few pixels in each round (somehow by rules that depend on $n$). There is a co-evolution of \textit{all} the (training data) images, and this in tandem with the weights of an autoencoder that seeks to fix all the images. Each autoencoder's weights are not refinements of the weights of the previous autoencoder, but are learned from scratch (random initialization) in each round. It is gratifying that even with this indirect protocol, the ``fitness'' of sets exhibits a systematic increase (Figs. \ref{fig:booLNIST24}-\ref{fig:booLNIST12}). Finally, we should not be dismayed to see 8's evolving into 3's, or 4's into 9's. The point of the exercise was to estimate entropy, not to read zip-codes!

Even with about the same level of average bit reconstruction accuracy, the two architectures have quite different distributions of the Hamming distance to the reconstructed data. This is shown in the left panel of Figure \ref{fig:booLNISTham} for the most evolved sets when $n=24$ and $n=12$. The autoencoder with the smaller waist has the greater fraction of data reconstructed perfectly, 86.9\%, but the distribution also has a long, low frequency tail extending to large Hamming distance. Doubling the waist of the autoencoder halves the rate of perfect reconstruction, 42.9\%, and there is no long tail.

The most significant effect of the size of the autoencoder waist is the rate of rejections by the autoencoder, of data generated by sampling the inputs of the encoder. This is shown in the right panel of Figure \ref{fig:booLNISTham}. With $d>0$ as the criterion for rejection, about $\alpha=61.2\%$ of the encoder-generated samples pass the set-membership test when $n=24$. This rate drops to $\alpha=6.4\%$ in the narrow-waisted autoencoder. The corresponding model entropies, $n-\log_2(1/\alpha)$, are respectively 23.3 and 8.0 bits. A glance at the accepted encoder-generated images, shown in Figure \ref{fig:booLNISTgen}, indicates that the lower entropy value is closer to the ``truth''. Most of the images generated by the $n=24$ encoder/autoencoder have such low connectivity (of black pixels) they would not even pass as imitations of something handwritten. 

\begin{figure}
    \centering
    \includegraphics[width=\columnwidth]{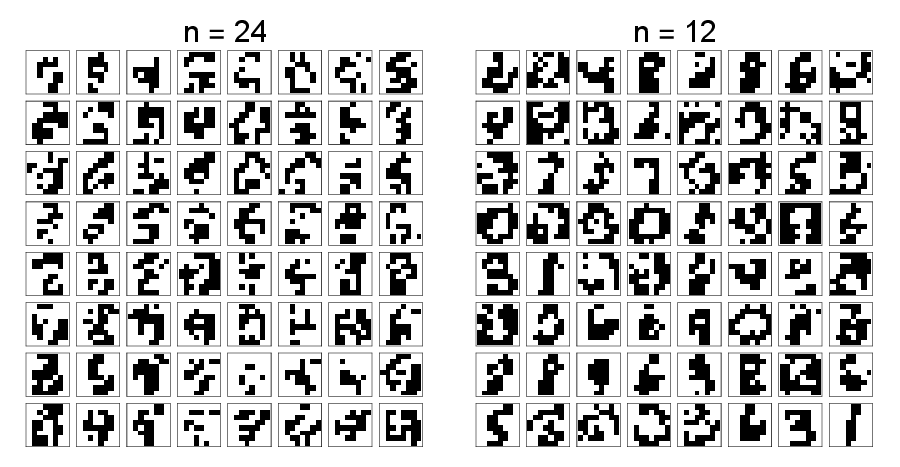}
    
    \caption{Images generated by uniformly sampling the encoder's inputs and retaining only those outputs that are  fixed by the autoencoder. The set evolved with the $n=12$ autoencoder (right) at least respects the property that the black pixels are well-connected.}
    \label{fig:booLNISTgen}
\end{figure}

\section{Discussion}\label{sec:discussion}

When children acquire language we can be sure they are not minimizing risk. The data samples are far too small for the brain to collect meaningful ``statistics''. We could cite Chomsky, but an articulate 4-year-old known to one of us summed it up nicely. Asked whether she is partly responsible for the recent explosion of new words babbled by her $1\,\frac{1}{2}$-year-old sister, her reply: ``No, she just knows.’’

This anecdote is meant to remind us that in spite of the paradigm-shift brought on by recent advances in machine learning, there is still much we do not understand. Accordingly, our concluding remarks are mostly questions, not answers.

``Generalization'' has never been as fiercely debated as it has in the AI era, perhaps because it has become possible to ``do'' it (not just talk about it). No one disputes that a \textit{hypothesis space} should lie at the core of its definition. However, opinions differ on the ``doing''---how the space is implemented. At one extreme is a structural space, like the conjunctive normal form expressions in the debut of PAC learning \cite{valiant1984theory}. At the other extreme is to let the learning dynamics play a role in its definition, as happens when using gradient-descent on over-parameterized deep networks. In our work---machine learnable sets---the hypothesis space is explicitly structural, and yet is implemented with deep networks. Are dynamically defined hypothesis spaces better at generalization, or are they just the only game in town when one is limited to gradient-based optimization?

Purely structural hypothesis spaces seem hopelessly optimistic when confronted with messy, noisy data. Statistical learning theory offers a lifeline by making \textit{distributions} the target of learning. A brilliant example of the method is the variational autoencoder (VAE) \cite{diederik2019introduction}. This starts out with the intractability of computing the marginal of a distribution we wish to express in terms of a joint distribution with latent variables. A neural network can be used to express complex distributions of the data, conditional on the latent variables. But to optimize the neural network's parameters---to increase the log-likelihood of the data---one needs the computationally intractable marginal distribution, or by Bayes' rule, the unknown posterior distribution of the latent variables. The VAE is the clever solution where another neural network is used to express the posterior, and Jensen's inequality is used to conjure a loss that is minimized when the ``variational'' posterior matches the true posterior. The resulting model is an autoencoder that includes loss terms associated with the latent variables.

But how sound, structurally, is the learning achieved by VAEs? A popular demonstration uses the MNIST data and just two latent variables at the autoencoder's waist. A Gaussian is the universal choice for the prior distribution on the latent variables, for which the associated VAE loss translates to a quadratic cost that tries to keep each digit's latent variables close to the origin. The autoencoder's reconstruction loss is minimized when similar digits have similar latent variables. But the quadratic cost of the Gaussian prior keeps the latent distribution of all the digits compact, so regions of similar digits tile the plane near the origin without wasting space. The region of 8's shares a boundary with the region of 3's, along which 8's morph continuously into 3's. Clearly this is not something we want the machine to learn, as it disrespects topology (or the difference in how these digits are rendered, by a hand, in time). The culprit, of course, is the Gaussian prior which turned out to impose the wrong structure on the hypothesis space.

Statistical learning theory is also used to model the progress of learning. By how much should we expect the test accuracy to increase as the number of training items is increased? The PAC model \cite{valiant1984theory} delivers a rigorous answer to this question, but only when the hypothesis space is so simple a relatively trivial learning algorithm exists. Moreover, the improvement in the (expected) accuracy is gradual and not at all like our experience with language acquisition by children. But the abrupt onset of generalization---like a phase transition in statistical mechanics---is exactly what we observe in machine-learnable sets (Fig. \ref{fig:RS4trans}, right panel). Is this because the constraint-based learning algorithm  in these experiments (RRR) is far from trivial, apparently creating the conditions for a phase transition when the number of constraints crosses a critical threshold?

Research on deep networks has been light on investigations of \textit{elementary mechanisms}. We can compare it to a semiconductor industry created without the principles of quantum mechanics. In the absence of a comprehensive theory of nonconvex optimization, such investigations will be experimental. But surely there are better experiments---for investigating the formation of neural pathways---than telling trousers from t-shirts \cite{xiao2017fashion}? The experiments and data favored by industry may not be the best for understanding elementary phenomena. As a complement, we encourage researchers to have their deep networks take ``the Rorschach test'' (Section \ref{sec:RS}). This is not to diagnose schizophrenia, but to see what kinds of neural pathways develop that detect this symmetry. We have studied this phenomenon in networks of Boolean threshold functions (boolnets) and it is quite interesting. It would be a shame if standard networks did not exhibit something similarly elegant.

New protocols for machine learning are also worth exploring, again because they may offer fresh insights on hotly debated topics. If nothing else, machine-learnable sets give us a new way of representing and propagating information. An object of impractically large size—the set itself—is defined in terms of two quite finite things: a small fragment of the object and complexity bounds on some Boolean circuits. From the fragment one can learn the settings of the switches in the Boolean circuits that define the whole. True machine-learnable sets are propagated among agents just by the act of learning from fragments.

Are machine-learnable sets singularly special points in the universe of objects one cares to learn? Our experiments lead us to believe the answer to this question is ``no''. The main evidence is the evolution of imperfect sets by rule \eqref{eq:setevolve} in Section \ref{sec:randencoder}, where two imperfect sets were generated by random 4-layer encoders. After only two rounds of evolution the two sets and their autoencoders had evolved enough to boost their self-recognition accuracy (on test samples) above 99\% (Fig. \ref{fig:wildevo}, right panel), while at the same time never confused the other's elements as their own. The encoders used to seed those two sets had so many random parameters, the number of near-perfect machine-learnable sets that can be generated in this way is practically unlimited.

\begin{acknowledgments}
V.E. thanks Shikang He and Yizhou Li for feedback on an early definition of the eponymous sets.
\end{acknowledgments}

\subsection*{Responsible research statement}

AI was not used in any part of this work (concept, programming, writing, artwork) with one exception: We asked MS-Copilot to search the literature for work with ``significant conceptual overlap'', that is, similarity beyond shared keywords. That search produced two publications~\cite{kirby2008cumulative,vandenoord2017neural}, though nothing resembling machine-learnable sets.

No animal testing was conducted for this study and no functions were differentiated.

\bibliographystyle{unsrtnat}
\bibliography{MLS}

\appendix

\section{Soft boolnets?}

Readers interested in experimenting with machine-learnable sets but not
ready to adopt the constraint-based methodology, will need to consider ``soft'' alternatives
that are amenable to gradient optimization. After discussing the challenges this creates, we present preliminary findings for an approach where $\sgn(\;)$ is replaced with $\tanh(\;)$ and weight normalization is imposed with a penalty function.

A soft autoencoder built with ReLU activation functions and suitably overparametrized might do quite well if the only task is to accurately reconstruct the inputs at the outputs (even when these are $\pm 1$ valued). But machine-learnable sets make two additional demands on the model. First, the values at the autoencoder's waist should also be discrete ($\pm 1$) in order to satisfy criterion 2 of machine-learnable sets (uniform sampling). Second, the decoder and encoder should fall into function classes whose complexity can be bounded (ideally like a Boolean circuit). These demands indicate that at a minimum \textit{all} the activation functions should saturate at two values, like $\tanh(\;)$.

In boolnets the complexity is additionally controlled by gate-input sparsity. Through the combination of a lower bound on the Boolean threshold function margins, and a constraint on the norms of the weights responsible for these margins, the BTFs in the model are equivalent to few-input logic gates (rigorously, under the diversity hypothesis \cite{elser2026learning}). There are soft counterparts to such constraints but so far our efforts using them have come up short, for reasons well known to anyone who has used convex optimization on highly non-convex problems.

The models learned by gradient methods are determined by the basin of the initial point. If the models correspond to circuit-like function classes, the basins will have a combinatorial character, only some of which are able to exactly reconstruct the autoencoder's outputs from its inputs. When the initial point falls in a bad basin, the only recourse is another training attempt. Though this is a viable approach when the rate of successes is not too small, it stands in sharp contrast to the demonstrated reliability of the constraint-based method. We encourage developers of soft boolnets to evaluate their methods with ``the Rorschach test''. The results for $8\times 8$ data, in Figure \ref{fig:RS4runs}, provide a point of comparison. Not only does the non-gradient method succeed in every attempt (two are shown), the progress of learning (gap and accuracy vs. iterations) is reproduced in detail across attempts.

A soft formulation we found that manages to find exact solutions for $8\times 8$ Rorschach data on some fraction of attempts uses the activation function
\begin{equation*}
y=\tanh\big((w\cdot x+b)/\mu\big)\;,
\end{equation*}
where $\mu$ controls the width of the transition between $-1$ and $+1$ relative to the neuron's linear input. The $\mu$ hyperparameter cannot be eliminated by rescaling the $w$'s and $b$'s of the model because we also impose the following penalty on the squared-norms of the neurons' weights:
\begin{equation*}
L_w=\lambda_w\sum_n\big(\|w_n\|^2-m_n \big)^2\;.
\end{equation*}
Here $\lambda_w>0$ is the weight-norm-penalty hyperparameter, and $m_n$ is the number of inputs to neuron $n$ (the width of the layer below neuron $n$). Changing the target values of the squared-norms is equivalent to a change in $\mu$. The choice $\|w_n\|^2=m_n$ is natural because $w_n$ and $x_n$ appear symmetrically in the preactivation, while $\|x_n\|^2=m_n$ when $x_n$ is a $\pm 1$ vector. Our model's loss function combines the weight-norm penalty with the autoencoder's reconstruction loss over all training data:
\begin{equation}\label{eq:loss}
L=L_w+\sum_{x\in X^*} \sum_i\big(\mathcal{A}(x)_i-x_i\big)^2\;.
\end{equation}

In our $8\times 8$ Rorschach experiments we used the same autoencoder architecture
\begin{equation*}
64\;\to\;36\;\to\;96\;\to\;64
\end{equation*}
that was used with the constraint method. The loss function \eqref{eq:loss} for 500 training data was minimized with the ADAM optimizer \cite{kingma2014adam} at its default settings. The constraint method easily achieves 100\% train and test accuracy with this much data (Fig. \ref{fig:RS4trans}). The success rates, in 20 trials, for the gradient method are tabulated in Table \ref{tab:successrates} for a range of the hyperparameters $\mu$ and $\lambda_w$ and a limit of $10^5$ on the number of gradient steps. We give the soft method the benefit of the doubt by designating an attempt successful if at least $99\%$ of the autoencoder's input and output bits (across 500 training data) are in agreement. Thus a table entry of $50\%$ means that $10$ of the $20$ trials met the $99\%$ accuracy threshold.

\begin{table}[t]
\centering
\caption{Success rates for learning the $8\times 8$ Rorschach set with different $\mu$ (rows) and $\bar\lambda_w=10^4\;\lambda_w$ (columns).}
\label{tab:successrates}
\setlength{\tabcolsep}{16pt} 
\begin{tabular}{r|rrr}
& $\bar\lambda_w=0.5$ & $1$ & $1.5$ \\
\hline
$\mu=6$ & $25\%$ & $35\%$ & $30\%$ \\
$8$ & $40\%$ & $60\%$ & $55\%$ \\
$10$ & $35\%$ & $45\%$ & $25\%$ \\
\end{tabular}
\end{table}

That any of the entries in  Table \ref{tab:successrates} are above 0\% is encouraging. However, much work needs to be done. The interaction between the steepness of $\tanh(\;)$ and the weight-norm penalty that expands the sizes of the ``good'' basins is a black box. The mechanism at work needs to be understood in order that settings for $(\mu,\lambda_w)$---that work well for different architectures and data types/sizes---can be determined without trial and error.

\end{document}